\documentclass[10pt,journal]{IEEEtran}
\ifCLASSINFOpdf
  \usepackage[pdftex]{graphicx}
\else
  \usepackage[dvips]{graphicx}
\fi
\usepackage{amsmath}
\usepackage{algorithmic}
\usepackage{array}
\ifCLASSOPTIONcompsoc
  \usepackage[caption=false,font=normalsize,labelfont=sf,textfont=sf]{subfig}
\else
  \usepackage[caption=false,font=footnotesize]{subfig}
\fi

\usepackage{stfloats}
\usepackage{url}
\usepackage[numbers, compress]{natbib}

\usepackage[utf8]{inputenc} % allow utf-8 input
\usepackage[T1]{fontenc}    % use 8-bit T1 fonts
\usepackage[pdftex, pdftitle={Learning Waveform-based Acoustic Models usig Deep Variational Convolutional Neural Networks}, pdfauthor={Dino Oglic, Zoran Cvetkovic, and Peter Sollich}, colorlinks=true, bookmarksnumbered, citecolor=darkblue, urlcolor=darkblue, linkcolor=darkblue, pagecolor=darkblue, pdfpagemode=none]{hyperref}       % hyperlinks

\usepackage{booktabs}       % professional-quality tables
\usepackage{amsfonts}       % blackboard math symbols
\usepackage{nicefrac}       % compact symbols for 1/2, etc.
\usepackage{microtype}      % microtypography

\usepackage{amssymb}
\usepackage{amsthm}
\usepackage{thmtools,thm-restate}
\usepackage{mathtools}

\usepackage{makeidx}

\usepackage{calc}
\usepackage{multirow}
\usepackage{makecell}

\usepackage{wrapfig}
\usepackage[font=footnotesize, skip=2pt]{caption}

\usepackage{wrapfig}

\usepackage{color}
\usepackage{xcolor}
\usepackage{colortbl}
\definecolor{lb}{RGB}{155,206,227}
\definecolor{db}{RGB}{31,120,180}
\definecolor{lg}{RGB}{178,223,138}
\definecolor{dg}{RGB}{51,160,44}
\definecolor{darkblue}{rgb}{0,0,0.75}

\definecolor{caddendum}{RGB}{0,0,0}
\definecolor{cerratum}{RGB}{0,0,0}
\definecolor{c-lss-1}{RGB}{28,144,153}
\definecolor{c-lss}{RGB}{255,0,0}
\definecolor{c-dpp}{RGB}{127,121,73}
\definecolor{c-dpp-1}{RGB}{127,121,73}
\definecolor{c-km}{RGB}{117,107,177}
\definecolor{lloyd-km}{RGB}{241,163,64}

\newtheorem{theorem}{Theorem}
\newtheorem*{theorem*}{Theorem}

\DeclareMathOperator*{\argmax}{arg\,max}

\DeclareMathOperator*{\kl}{\mathrm{KL}}

\newcommand{\D}[1]{\,\mathrm{d}{#1}}
\newcommand{\brackets}[1]{\left( {#1} \right)}
\newcommand{\Brackets}[1]{\Big( {#1} \Big)}
\newcommand{\cbrackets}[1]{\left\{ {#1} \right\} }
\newcommand{\sbrackets}[1]{\left[ {#1} \right]}
\newcommand{\SBrackets}[1]{\Big[ {#1} \Big]}

\newcommand{\norm}[1]{\left\Vert {#1} \right\Vert}
\newcommand{\absolute}[1]{\left\vert {#1} \right\vert}

\newcommand{\ceq}[1]{(\ref{#1})}
\newcommand{\acro}[1]{\textsc{\MakeLowercase{#1}}}

\newlength\figureheight
\newlength\figurewidth

\usepackage{tikz}
\usepackage{pgfplots}
\usepackage{pgf}
\usepackage{adjustbox}
\usepackage{pgfplotstable}
\pgfplotsset{compat=newest}
\pgfplotsset{
	tick label style={font=\small},
	label style={font=\small},
	legend style={font=\small},
	every axis/.append style={
		thick,
		tick style={semithick, black},
		axis line style={-},
		axis x line =bottom,
		axis y line =left
	}
}

\usepgfplotslibrary{groupplots}
\usepgfplotslibrary{fillbetween}
\usepgfplotslibrary{external} 
\usepackage{soul}

\begin{document}
\title{Learning Waveform-Based Acoustic Models using Deep Variational Convolutional Neural Networks}

\author{Dino~Oglic,
        Zoran~Cvetkovic,
        and~Peter~Sollich% <-this % stops a space
\thanks{D. Oglic and Z.\ Cvetkovic are with the Department of Engineering, King's College London. Correspondence to: $\texttt{dino.oglic@uni-bonn.de}$.}% <-this % stops a space
\thanks{P.\ Sollich is with the Department of Mathematics, King's College London, and the Institute for Theoretical Physics, University of G\"ottingen.}% <-this % stops a space
}


\markboth{ IEEE/ACM Transactions on Audio, Speech, and Language Processing}%
{Shell \MakeLowercase{\textit{et al.}}: Bare Demo of IEEEtran.cls for IEEE Journals}

\IEEEpubid{2329-9290~\copyright~2021~IEEE. Citation information: \href{https://ieeexplore.ieee.org/document/9511850}{DOI~10.1109/TASLP.2021.3104193}, IEEE/ACM Transactions on Audio, Speech, and Language Processing.}


\maketitle
\IEEEpeerreviewmaketitle

\begin{abstract}
We investigate the potential of stochastic neural networks for learning effective waveform-based acoustic models. The waveform-based setting, inherent to fully end-to-end speech recognition systems, is motivated by several comparative studies of automatic and human speech recognition that associate standard non-adaptive feature extraction techniques with information loss, which can adversely affect robustness. Stochastic neural networks, on the other hand, are a class of models capable of incorporating rich regularization mechanisms into the learning process. We consider a deep convolutional neural network that first decomposes speech into frequency sub-bands via an adaptive parametric convolutional block where filters are specified by cosine modulations of compactly supported windows. The network then employs standard non-parametric \acro{1d} convolutions to extract relevant spectro-temporal patterns while gradually compressing the structured high dimensional representation generated by the parametric block. We rely on a probabilistic parametrization of the proposed neural architecture and learn the model using stochastic variational inference. This requires evaluation of an analytically intractable integral defining the Kullback--Leibler divergence term responsible for regularization, for which we propose an effective approximation based on the Gauss--Hermite quadrature. Our empirical results demonstrate a superior performance of the proposed approach over comparable waveform-based baselines and indicate that it could lead to robustness. Moreover, the approach outperforms a recently proposed deep convolutional neural network for learning of robust acoustic models with standard \acro{fbank} features.
\end{abstract}


% \fi
\IEEEpeerreviewmaketitle

\section{Introduction}
\label{sec:intro}

Automatic speech recognition systems typically operate in low-dimensional feature spaces designed to achieve invariances inherent to speech production and human speech recognition~\cite{li12,plos,tuske2018}. Log Mel-filter bank values (\acro{fbank}) and their de-correlated variant known as Mel-frequency cepstral coefficients (\acro{mfcc}) are two of the most frequently used feature extraction techniques of this kind~\cite[][]{bridle1974experimental,mfcc1980}. Several comparative studies of automatic and human speech recognition~\cite{alsteris,meyer07,petere99} suggest that the information loss inherent to such feature extraction techniques can adversely affect robustness to standard environmental distortions arising from additive and channel (linear filtering) noise~\cite{Ager11,Yousafzai11a}. Motivated by this, we propose an effective and principled approach for learning of robust acoustic models in the waveform domain. A difficulty in the waveform setting is the sheer size of the training data required for learning effective waveform-based models. More specifically, the requirement for more than $2,000$ hours of speech in~\cite[][]{sainath2015,zhu2016} translates into weeks of training on a typical device with \acro{gpu} support. 
Our aim is to tackle this problem by incorporating relevant inductive bias into the learning process and allow for learning of effective waveform-based acoustic models using moderately sized datasets. 
There are two components in our approach, one dealing with the design of neural architectures and the other with learning of the corresponding parameters. 

Section~\ref{sec:parznets} is concerned with the design of neural architecture, 
which should perform automatic feature extraction by avoiding fast compression schemes associated with information loss when operating with standard non-adaptive filterbank features~\cite{alsteris,meyer07,petere99}. We design the neural network as a Lipschitz continuous operator that maps speech waveform frames into a feature space in such a way that small perturbations in the inputs caused by local translations and diffeomorphisms result in relatively small changes in the pre-softmax network outputs. As we operate in the waveform domain, the first layer of our convolutional network extracts information relevant for discrimination between phonetic units by decomposing a speech frame into frequency sub-bands using a set of parametric band-pass filters. The filters are defined by cosine modulations of compactly supported windows and allow for embedding of waveform signals into a structured high-dimensional space where we hypothesize that phonetic units will be easier to separate. The network then employs standard \acro{1d} convolutional layers with non-parametric filters for extraction of relevant spectro-temporal patterns while gradually compressing the structured representation generated by the sub-band decomposition. The outputs of the last such convolutional block are passed to a multi-layer perceptron (\acro{mlp}) with a softmax output.\IEEEpubidadjcol

The learning component of our approach is described in Section~\ref{subsec:svi}. We propose to learn a probabilistic parametrization of our architecture using variational inference. The motivation for this comes from the fact that for robustness one needs to be able to select the operator mapping with a good Lipschitz constant. The role of probabilistic parametrization and variational inference is to regularize the training process, thus allowing us to learn a robust feature representation of speech signals. This is different from a typical acoustic model, which employs an artificial neural network with real-valued parameters. Such a \emph{deterministic} parametrization of the network fails to capture the uncertainty of individual parameters and their importance for the learning task. Bayesian machine learning provides a principled framework for modeling uncertainty by finding plausible models that could explain the observed data~\cite{barber,Ghahramani2015}. In particular, a (deterministic) neural network with fixed parameter values models the conditional probability of a sub-phonetic unit given a speech frame. In \emph{stochastic} neural networks one additionally assumes that the parameters follow some prior distribution. The latter coupled with the aforementioned likelihood gives rise to a posterior distribution of parameter values conditioned on the observed data. Such posteriors are typically defined via analytically intractable integrals that can be approximated using scalable inference techniques such as stochastic variational inference~\cite{blundell15,Buntine91,Graves}. In particular, the main idea is to approximate intractable posteriors by optimizing over parameters of an a priori selected family of variational distributions. The optimization objective in variational inference consists of two terms: \emph{i}) the expected negative log-likelihood of the model, where the expectation is taken with respect to the variational distribution, and \emph{ii}) the Kullback--Leibler divergence that performs  regularization. The expectation in the first term is approximated by sampling the variational distribution, which is typically given by a Gaussian mean field. In this way, the variational formulation injects randomness into the forward pass that computes the loss associated with a particular mini-batch. As a result, stochastic neural networks can capture parameter uncertainty and are less sensitive to perturbations in parameter values, as well as less susceptible to over-fitting~\cite{blundell15,Graves}. A further regularization effect, 
incorporated via the Kullback--Leibler divergence, is specified by an analytically intractable integral. For this we propose an effective approximation based on the Gauss--Hermite quadrature. Variational inference has been used previously in speech recognition, albeit in a different context, to maintain the balance between a dataset size and model complexity~\cite{watanabe2012,watanabe03}. In addition to this, a high correlation between the uncertainty in individual parameters and their importance for speech recognition has been observed in stochastic recurrent nets~\cite{braun,Graves}. Previous work, however, does not operate in the waveform domain, focuses on recurrent nets and considers variational inference separately from the properties encoded into the architecture (i.e.\ Lipschitz continuity in our case).

In Section~\ref{sec:rw}, we focus on the relationship with prior work on speech recognition in the waveform domain. We then evaluate the proposed approach empirically on three benchmark datasets for automatic speech recognition: \acro{timit}, \acro{aurora4}, and \acro{ami-ihm}. A summary  of our empirical results is provided in Section~\ref{sec:experiments}. 
The ablation study (evaluating the effectiveness of individual components in our approach) demonstrates that acoustic models based on modulation filter learning can be more effective, in a statistically significant way, than the ones with non-adaptive filters. Moreover, the experiments indicate that the proposed approximation scheme based on the Gauss--Hermite quadrature provides a general (with respect to the choice of prior function) and effective means for approximating the Kullback--Leibler divergence term. The experiments on the \acro{timit} dataset demonstrate that the approach does not over-fit despite using a rather large network on what in speech recognition is considered to be a small dataset. Moreover, our results on \acro{aurora4} show that the approach is capable of learning a noise robust model, outperforming significantly the state-of-the-art baselines for waveform-based speech recognition on this dataset. It is also promising that on the same dataset the approach outperforms a recently proposed deep convolutional network for learning of robust acoustic models with standard \acro{fbank} features~\cite{vdcnn}. The experiments on \textsc{ami} (conversational speech, without $\mathrm{i}$-$\mathrm{vectors}$ or data augmentation) show that the approach outperforms recently proposed architectures for raw speech (see~{\cite{multispan}} and~{\cite{sincnet}}) and performs on par with a state-of-the-art \textsc{fbank}/\textsc{mfcc} based deep time-delay neural network (\acro{tdnn}) model~\cite{Ghahremani2016}. Thus, our empirical contributions provide comprehensive evidence for the effectiveness of variational neural networks operating directly in the waveform domain.

\begin{figure*}[t!]
	\centering
	\includegraphics[scale=0.2]{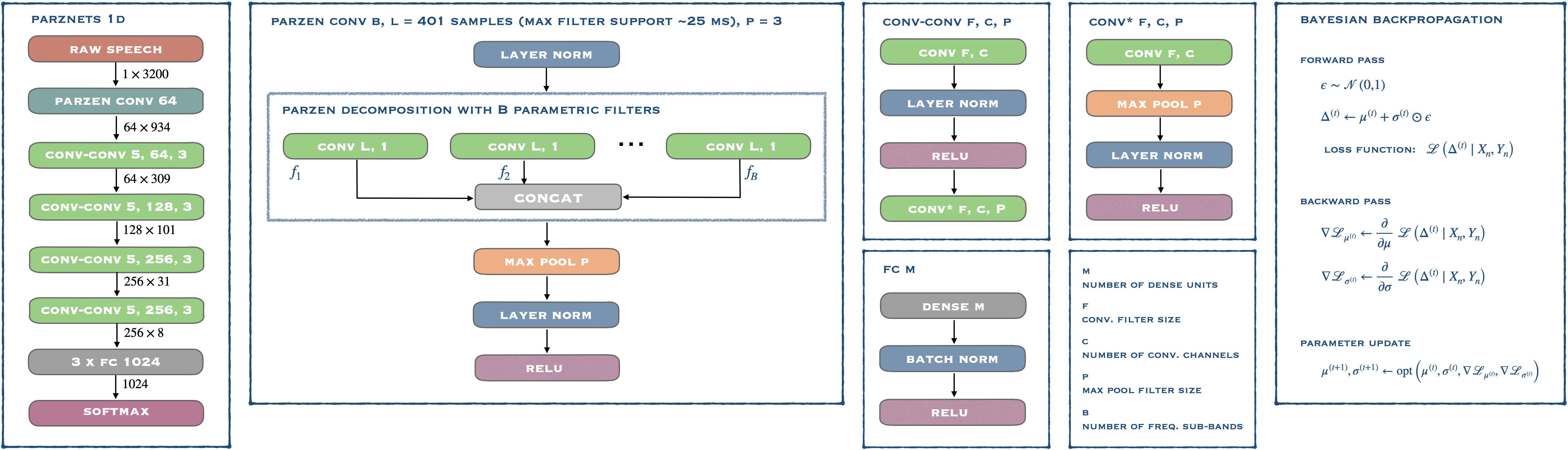}
	\caption{This is a schematic illustrating \acro{parznets} with \acro{1d} convolutional operators. The illustration is supplemented with the Parzen block (the second panel on the left) that decomposes a raw speech frame into frequency sub-bands and a pseudo-code description of Bayesian backpropagation used in variational inference. The loss function in the rightmost panel refers to the variational objective that is described in detail in Section~{\ref{subsec:svi}}.
	}
	\label{fig:arch}\vspace{-2ex}
\end{figure*}

\section{Parznets --- Deep Convolutional Neural Networks for Waveform-based Speech Recognition}
\label{sec:parznets}

This section describes an artificial neural network for learning
acoustic models in the waveform domain. We first provide a brief overview of the relevant building blocks of the architecture (Section~\ref{subsec:parzen-overview}) and then introduce a parametric convolutional layer responsible for decomposition of speech signals into frequency sub-bands (Section~\ref{subsec:parzen-filters}). The section concludes with a theoretical analysis demonstrating that the proposed neural architecture defines a Lipschitz continuous operator in the waveform domain (Section~\ref{subsec:parzen-motivation}).

\subsection{Overview of the Neural Architecture}
\label{subsec:parzen-overview}

We would like to design an architecture capable of embedding redundancies into the representation, thereby avoiding significant overlaps between positioning of different phonetic units while allowing for a fair amount of additive noise and distortion at inputs. Motivated by this, we extract information relevant for discrimination between phonetic units via a parametric Parzen convolutional block (Section~\ref{subsec:parzen-filters}) that decomposes a waveform frame into frequency sub-bands, thereby embedding the signal into a high-dimensional space of high-resolution spectro-temporal patterns (illustrated in Fig.~\ref{fig:arch}, \acro{parznets 1d}). A notable difference compared to non-adaptive feature extraction operators (\acro{fbank} and \acro{mfcc}) is the use of a \acro{reLU} activation function instead of the modulus (squared) non-linearity. Mallat~\cite{mallat16} has demonstrated that this change in activation function does not affect the theoretical properties of such operators. Moreover, it has been established recently that neural networks with \acro{relu} activations realize piecewise linear functions and we therefore use that non-linearity throughout the network~\cite{Croce2020Provable}. The main motivation behind this choice is to avoid further confounding effects between signal and noise that might otherwise arise from additional sources of non-linearity in the automatic feature extraction process (it is well known, for example, that the effects of channel noise can be amplified by non-linearities).   
To extract relevant patterns from such a sub-band decomposition/representation, we rely on standard non-parametric convolutional filters and pass the Parzen sub-bands to double convolutional blocks with $5$ sample long filters (see \acro{conv-conv} in Fig.~\ref{fig:arch}). The gradual compression of the spectro-temporal representation is achieved by applying the max pooling operator with size $3$ (after each pair of non-parametric convolutional blocks). Previous work~\cite{Goodfellow2016} has demonstrated that a composition of convolution with max pooling tends to provide approximate local time-translation invariance. In our preliminary experiments, we investigated the effectiveness of $\max$ and (weighted) $\ell_p$ average pooling operators, and observed that the former works the best in combination with \acro{relu} activations. The features extracted by the last convolutional block are passed to an \acro{mlp} block with three hidden layers (i.e., fully connected layers denoted by \acro{fc} in Fig.~\ref{fig:arch}), followed by a softmax output block.

\subsection{Parzen Block for Sub-band Decomposition of Speech Signals}
\label{subsec:parzen-filters}

It has been demonstrated recently that feature extraction operators that combine band-pass filtering with the modulus (square) non-linearity and (weighted) local averaging are approximately locally translation invariant and Lipschitz continuous~\cite{mallat14}. 
A potential shortcoming of these operators is the fact that filter parameters are selected a priori without relying on data. As a result, the hypothesis space is selected beforehand and does not necessarily provide an ideal inductive bias for all learning tasks. Moreover, the power spectral averaging that is characteristic of these operators is typically performed over speech segments of $25$ or $32$ ms~\cite{mallat14,mallat2012}, which could be compressing the relevant information too fast into the resulting features. As a result of such compression, the feature extraction operator might be discarding the information relevant for robustness. Motivated by this, we have designed the Parzen convolutional block to tackle these shortcomings. In particular, the block does not rely on a priori selected filters but learns these via parametric convolutions that have a strongly encoded inductive bias. Moreover, the adaptive Parzen convolutional block embeds a waveform frame into a structured high dimensional space rather than compressing it into a small number of features. The latter is an important difference compared to \acro{mfcc} and \acro{fbank} coefficients, which do not focus on embedding redundancies into the representation. As explained above, the Parzen sub-band decomposition is followed by a gradual compression of the representation using a combination of convolutional and max pooling operators.

In speech recognition, band-pass filtering of signals is traditionally performed by (weighted) averaging of power spectra~\cite[see][]{mfcc1980,gales07} computed over speech frames of fixed duration. 
Alternatively, the signal can be convolved with a filter directly in the time domain. To that end, we consider a family of differentiable band-pass filters based on cosine modulations of compactly supported Parzen windows~\cite{parzen1962}. In particular, we employ the squared Epanechnikov window function given by
\begin{align}
k_{\gamma}\brackets{t} = \max \cbrackets{0, 1 - \gamma t^2}^2 \ ,
\end{align}
where $\gamma$ is a parameter controlling the window width, and implicitly its frequency bandwidth. The filter can be made more frequency selective by increasing its exponent (illustrated above with the square operator), which is a consequence of increasing its order of differentiability. To allow for flexible placement of the center frequency we  rely on cosine modulation. Thus, Parzen filters are defined with only two differentiable parameters, $\eta$ controlling the modulation frequency and $\gamma$ controlling the filter bandwidth:
\begin{equation}
\label{eq:parzen}
\phi_{\eta,\gamma}\brackets{t}=\cos \brackets{2\pi \eta t} \cdot k_{\gamma}\brackets{t} \ . 
\end{equation}
As illustrated in Fig.~\ref{fig:arch} (the leftmost panel), for each filter configuration $\cbrackets{\brackets{\eta_i, \gamma_i}}_{i=1}^B$, we use Eq.~\ceq{eq:parzen} to generate a one-dimensional convolutional filter with maximum length given by the number of samples in $25$ ms of speech; filters with shorter support are symmetrically padded with zeros.
The outputs of parametric convolutions are concatenated into a high dimensional spectro-temporal decomposition of a signal and then passed to a max pooling operator, followed by layer normalization~\cite{layernorm}. As all of the operations in this parametric block are differentiable, it is possible to construct an auto-differentiation graph that seamlessly provides gradients with respect to parameters of Parzen filters. In comparison to wavelet filters~\cite{Khan18NIPS}, the Parzen convolutional block offers additional flexibility by allowing independent control over bandwidth and modulation frequency. Moreover, the block optimizes for the positioning of the two parameters while having the parametric form of the filter factored into the optimization. This can be seen as a more flexible approach compared also to the two-step procedure employed by~\cite{sincnet}, where filter cut-off frequencies are optimized with respect to a fixed-length rectangular window, and then a Hamming window is superimposed to suppress the ripple effects.

\subsection{Lipschitz Continuity of the Operator Mapping}
\label{subsec:parzen-motivation}

We start with a review of Lipschitz continuity for operator mappings and properties relevant for their robustness. Following this, we demonstrate that the principle for the design of neural architectures outlined in Section~\ref{subsec:parzen-overview} and Fig.~\ref{fig:arch} defines a Lipschitz continuous operator in the waveform domain.

Let $\mathcal{L}\brackets{\mathbb{R}}$ denote the space of square integrable functions defined on $\mathbb{R}$ and consider a continuous signal $f \in \mathcal{L}\brackets{\mathbb{R}}$. An operator $\Phi \colon \mathcal{L}\brackets{\mathbb{R}} \rightarrow \mathcal{H}$ is a mapping of a signal into a Hilbert space $\mathcal{H}$. Let $T_c f \brackets{t}=f\brackets{t-c}$ denote the translation of a signal $f$ by some constant $c \in \mathbb{R}$. An operator $\Phi$ is called \emph{translation invariant} if $\Phi\brackets{T_cf}=\Phi\brackets{f}$ for all $f \in \mathcal{L}\brackets{\mathbb{R}}$ and $c\in \mathbb{R}$. The spectrogram of a signal is an operator designed to capture variations in the power spectrum over time. It can provide an approximately locally time-translation invariant representation over durations limited by a window~\citep{mallat14}. While the spectrogram of a signal can provide local time-translation invariance, Mallat~\cite{mallat2012} has demonstrated that it does not necessarily provide stability to the action of a small diffeomorphism (e.g., speed perturbation of an utterance). Let $D_{\tau} \colon \mathcal{L}\brackets{\mathbb{R}} \rightarrow \mathcal{L}\brackets{\mathbb{R}}$ be a diffeomorphism of a signal (i.e., invertible function that maps one differentiable manifold to another such that both the function and its inverse are smooth) given by $D_{\tau}f\brackets{t}=f\brackets{t-\tau \brackets{t}}$, where $\tau \brackets{t}\in \mathcal{C}^2\brackets{\mathbb{R}}$ is a displacement field and $\mathcal{C}^2\brackets{\mathbb{R}}$ denotes the space of twice continuously differentiable functions over the reals. For example, one can take $\tau\brackets{t}=\epsilon t$ with $\epsilon \in \mathbb{R}$ and $\epsilon \rightarrow 0$. To preserve stability relative to a small diffeomorphism of a signal, it is sufficient to ensure that the operator $\Phi$ is Lipschitz continuous~\citep{mallat2012,mallat14}. A translation invariant operator $\Phi$ is Lipschitz continuous with respect to actions of $\mathcal{C}^2$-diffeomorphisms if for any compact $\Omega \subset \mathbb{R}$ there exists a constant $L$ such that for all signals $f \in \mathcal{L}\brackets{\mathbb{R}}$ supported on $\Omega$ and all $\tau \in \mathcal{C}^2\brackets{\mathbb{R}}$ it holds that~\citep[for more details see, e.g.,][]{mallat2012}
\begin{align*}
\begin{aligned}
& \norm{\Phi\brackets{f}-\Phi\brackets{D_{\tau}f}}_{\mathcal{H}}\leq L \norm{\mathbb{I}-D_{\tau}}_{\infty}\norm{f} && \\
& \coloneqq  L \brackets{\sup_{t \in \Omega}\ \norm{\nabla\tau \brackets{t}} + \sup_{t \in \Omega}\ \norm{\nabla\nabla\tau \brackets{t}} } \norm{f} \ , & 
\end{aligned}
\end{align*}
where $\mathbb{I}$ denotes the identity mapping. The Lipschitz continuity of operator $\Phi$ implies invariance to \emph{local translations} and/or signal warping by a diffeomorphism $\tau\brackets{t}$, up to the first and second order deformation terms~\cite{mallat2012}. Such signal perturbations typically come as a result of variability in speech production and differences between speakers. Another aspect of robust representations is the ability to withstand a fair amount of additive and channel/linear noise. It is easy to show (e.g., using the convolution theorem) that such a perturbation of a clean speech signal amounts to a linear transformation of its representation in the frequency domain. Thus, an operator that is Lipschitz continuous over the sub-band decomposition of a signal has the potential to work effectively on noisy speech. In particular, a noise corrupted signal is a linear transformation of the clean signal in the frequency domain and will be contained within a ball of constant radius centered at the clean signal. An operator that is Lipschitz continuous over the frequency representation of a signal will exhibit small variations over such balls and, thus, it can provide stability relative to additive and channel noise.  
It is, however, important to point out that the robustness of such an operator quantitatively depends on the value of the Lipschitz constant.

The operator defined by our neural network maps a frame of raw speech $x \in \mathbb{R}^d$ into a vector of pre-softmax outputs $z \in \mathbb{R}^s$, where $d$ is the number of samples in the input frame and $s$ is the dimension of the pre-softmax representation. Moreover, this is achieved by having an intermediate representation of the signal in the frequency domain via sub-band decomposition performed by the Parzen block. The operator mapping can be expressed as a composition of functions
\begin{align*}
	\Phi \brackets{x} = \Brackets{\rho_l \circ \rho_{l-1} \circ \dots \rho_1}\brackets{x} \ ,
\end{align*}
where $\rho_i$ represents the $\acro{relu}$ activation function, linear or pooling operator. In particular, the building blocks of our architecture are fully connected and convolutional layers, which are both linear operators and can be realized as matrix-vector multiplications~\cite[see, e.g.,][]{Gouk18}. For a fully connected block with weights $W$ and bias $b$, the Lipschitz constant $L$ is given by
\begin{align*}
	\norm{Wz + b - Wz' - b}_2 \leq L \norm{z-z'}_2  . 
\end{align*}
Thus, the minimal value of the Lipschitz constant is equal to $L=\sup_{z\in\mathcal{B}} \nicefrac{\norm{Wz}_2}{\norm{z}_2}$, where $\mathcal{B}$ is a ball of constant radius containing all the layer inputs in its interior. The convolution blocks can also be realized via matrix-vector multiplications using doubly block circulant matrices~\cite{Gouk18}. Thus, a good Lipschitz constant can be obtained by keeping low the upper bounds on the weights in linear blocks and convolutional filters, while at the same time optimizing for the operator mapping such that the sub-phonetic units are linearly separable. 

Gouk et al.~\cite{Gouk18} have demonstrated that the \acro{reLU} activation function is Lipschitz continuous with constant one. This activation function is also monotonic and, thus, defines a contraction. The same holds for the $\max$ operator used for signal pooling, as demonstrated with the following proposition.

\begin{restatable}{proposition}{propMaxPool}
The max pooling operator is a Lipschitz continuous function with constant one.
\end{restatable} 

\begin{proof}
The $\max$ pooling operator can be expressed as $\iota(z)= \max_{1\leq j \leq k} \sigma_{i, j} (z)$, where $\sigma_{i, j}(z)$ is the $j$-th output of the $i$-th network layer that takes a vector $z$ as input. We will show that
\begin{align*}
\absolute{\max_{1\leq j \leq k} \sigma_{i,j} (z) - \max_{1\leq j \leq k} \sigma_{i,j} (z')} \leq \norm{\sigma_i(z)-\sigma_i(z')}_{2} \ .
\end{align*}
We can, without loss of generality, assume that $\iota(z) \geq \iota(z')$. Denote $j_0=\argmax_{1\leq j \leq k}  \sigma_{i,j} (z)$. Then, 
\begin{align}
\begin{aligned}
& \absolute{\iota(z) - \iota(z')} = \max_{1\leq j \leq k} \sigma_{i,j} (z) - \max_{1\leq j \leq k} \sigma_{i,j} (z') \leq & \\
& \sigma_{i, j_0} (z) - \sigma_{i,j_0} (z') \leq \max_{1 \leq j \leq k} \brackets{\sigma_{i, j}(z)-\sigma_{i, j}(z') }\leq & \\
& \norm{\sigma_{i}(z)-\sigma_{i}(z')}_{\infty} \leq \norm{\sigma_{i}(z)-\sigma_{i}(z')}_{2} \ .&
\end{aligned}
\end{align}
\end{proof}

As the proposed neural architecture is defined using a composition of Lipschitz continuous functions, the resulting operator mapping is also Lipschitz continuous. To ensure that the training procedure selects a good Lipschitz constant, we propose to use a probabilistic parametrization for our network and learn the corresponding parameters using stochastic variational inference, as described in the next section. 

While Lipschitz continuity of neural architectures has already been associated with robust representation learning~\cite[e.g., see][]{Gouk18}, this is the first work that provides an explanation for possible advantages of the filterbank over sample-based audio processing. In particular, in order to learn an effective (relative to longer time-shifts, additive and channel/linear noise) waveform-based representation of speech signals one can design the neural architecture as a Lipschitz continuous operator in the waveform-domain, with an intermediate representation in the frequency domain that can be realized using a sub-band decomposition of the signal (the Parzen block in our case).

\section{Learning Parznets using Stochastic Variational Inference}
\label{subsec:svi}

In deterministic neural networks, parameters/weights are real-valued and one performs inference by optimizing a loss function over them. Performing inference in stochastic/probabilistic neural networks, on the other hand, requires a posterior distribution over parameters given data~\cite{Graves}. For a fixed setting of weights, a deterministic neural network with \emph{softmax} outputs models the conditional probability of a categorical label $y \in \mathcal{Y}$ given an instance $x \in \mathcal{X}$ using an exponential family model~\cite[e.g., see][]{jaynes57,altun04}. In stochastic networks, it is further assumed that weights have a prior distribution $p_r \brackets{\Delta \mid \eta}$, where $\Delta$ denotes all the parameters in the network and $\eta$ are prior hyper-parameters.
The posterior distribution of neural network parameters conditioned on a set of \acro{iid} examples $\{(x_i, y_i)\}_{i=1}^n$ with $X_n=\cbrackets{x_i}_{i=1}^n$ and $ Y_n=\cbrackets{y_i}_{i=1}^n$ is typically given by an analytically intractable integral, with parameter-specific posterior probabilities $p\brackets{\Delta \mid X_n, Y_n}$ satisfying
\begin{align*}
\begin{aligned}
& \log p\brackets{\Delta \mid X_n, Y_n} \propto \log p_r \brackets{\Delta \mid \eta} + \sum_{i=1}^{n}\log p\brackets{y_i \mid x_i, \Delta} \ .&
\end{aligned}
\end{align*}
Variational inference~\cite{blundell15,Buntine91,Graves,kingma15} is a technique for the approximation of posterior distributions involving analytically intractable integrals. It works by introducing a family of variational probability density functions $q\brackets{\Delta \mid \mu, \sigma}$, with $\mu$ and $\sigma$ denoting variational parameters, such that a set of these specifies a family of probability distributions. Typically, the variational family is parametrically much simpler than the posterior distribution over network parameters $p\brackets{\Delta \mid X_n, Y_n}$. The main idea is to approximate the posterior $p\brackets{\Delta \mid X_n, Y_n}$ by optimizing a lower bound on the log-marginal likelihood of the model over the parameters of the variational distribution 
\begin{align}
\label{eq:svi}
\min_{q \in \mathcal{Q}} \ \kl\brackets{q \mid\mid p_r} - \sum_{i=1}^{n} \mathbb{E}_{\Delta \sim q\brackets{\Delta \mid \mu, \sigma}} \sbrackets{\log p \brackets{y_i \mid x_i, \Delta}} ,
\end{align}
where $\mathcal{Q}$ is a family of variational distributions specified by domains of parameters $\mu$ and $\sigma$. The Gaussian mean field approximation assumes that the variational distribution is the product of univariate Gaussian distributions, i.e.\ $q\brackets{\Delta \mid \mu, \sigma}=\prod_{i=1}^p \ \mathcal{N}\brackets{\Delta_i \mid \mu_i, \sigma_i^2} $, where $p$ is the total number of parameters in the model, $\Delta_i$ is the $i$-th component of the parameter vector $\Delta$, and $\mathcal{N}\brackets{\Delta_i \mid \mu_i, \sigma_i^2}$ is a univariate Gaussian distribution of $\Delta_i$ with mean $\mu_i$ and variance $\sigma_i^2$.

The expected log-likelihood of the model 
\begin{align*}
L_n\brackets{q}=\sum_{i=1}^{n} \mathbb{E}_{\Delta \sim q\brackets{\Delta \mid \mu, \sigma}} \sbrackets{\log p \brackets{y_i \mid x_i, \Delta}} 
\end{align*}
is analytically intractable and an evaluation of this expectation is required for the forward-pass when computing the loss function for a setting of the variational parameters $\mu$ and $\sigma$. Stochastic variational inference approximates this term in the forward-pass by sampling the variational distribution~\cite{kingma15}:
\begin{align*}
L_n\brackets{q} \approx \tilde{L}_m\brackets{q} = \frac{n}{m} \sum_{i=1}^{m} \log p \brackets{y_i \mid x_i, \Delta}  ,
\end{align*}
with $\Delta_j = \mu_j + \epsilon_j \sigma_j$ being a sample from $\mathcal{N}\brackets{\Delta_j \mid \mu_j,\sigma_j^2}$ given by $\epsilon_j \sim \mathcal{N}\brackets{\epsilon_j \mid  0, 1}$ ($1 \leq j \leq p$), and where $\cbrackets{\brackets{x_i, y_i}}_{i=1}^m$ is a mini-batch with $m$ random examples. As illustrated in Fig.~\ref{fig:arch} (the rightmost panel), the parameters of the neural network are populated with a random sample $\Delta$ drawn from the variational distribution and with that setting one computes the loss function for a particular mini-batch. The forward-pass sequence of actions is differentiable with respect to the variational parameters $\upsilon=\cbrackets{\brackets{\mu_i, \sigma_i} }_{i=1}^p$ and unbiased. Consequently, the gradient of this estimator is also unbiased and can be computed in the backward-pass by $\nabla_{\upsilon} L_n\brackets{q} \approx \nicefrac{n}{m}\sum_{i=1}^{m} \nabla_{\upsilon} \log p \brackets{y_i \mid x_i, \Delta}$, where the network parameters $\Delta$ originate from the forward-pass components and are given by $\Delta_j= \mu_j + \epsilon_j \sigma_j $. Thus, stochastic neural networks update the variational mean and variance parameters during gradient descent and use back-propagation for the computation of the gradients with respect to these parameters. At test time, the parameters of neural architecture are populated with variational means. In this way, a stochastic neural network injects randomness into network parameters for each mini-batch. As a result, the inferred model can capture parameter uncertainty and is likely to be more stable to parameter perturbations than an equivalent deterministic model. A further regularization effect can be achieved via the Kullback--Leibler divergence term (Eq.~\ref{eq:svi}), discussed in the next section.

\subsection{Approximation of Kullback--Leibler Divergence}
\label{subsec:svi-priors}

The Kullback--Leibler divergence term is responsible for regularization (Eq.~\ref{eq:svi}) and it is defined in terms of an analytically intractable integral that is typically approximated by Monte Carlo estimates using samples from the variational distribution~\cite{blundell15} or prior specific second order approximations~\cite{kingma15,molchanov17a}. We propose an approximation scheme based on the Gauss--Hermite quadrature, which independently of the prior distribution used allows for an approximation with a polynomial of arbitrarily high degree. More specifically, variational inference typically relies on Gaussian mean field approximations and this implies that the divergence term can be expressed as a sum of one dimensional integrals with respect to univariate Gaussian measures. Such integrals can be effectively approximated using the Gauss--Hermite quadrature~\cite[][]{AbramowitzStegun}, which is a quadrature with the weight function $\exp(-u^2)$ over the interval $u \in (-\infty, \infty)$. The following theorem provides a formal specification of the Gauss--Hermite quadrature for univariate functions.

\begin{theorem}~\cite[Abramowitz and Stegun,][]{AbramowitzStegun}
	For a univariate function $h$ and an integral
	\begin{align*}
	\mathcal{J}=\int_{-\infty}^{\infty} h\brackets{u}\exp\brackets{-u^2} \D{u} \ ,
	\end{align*}	
	the Gauss-Hermite approximation of order $s$ satisfies $ \mathcal{J}\approx \sum_{i=1}^{s}w_i h\brackets{u_i} $,
	where $\cbrackets{u_i}_{i=1}^s$ are the roots of the physicist's version of the Hermite polynomial 
	$	H_s\brackets{u}=\brackets{-1}^s\exp\brackets{u^2}\frac{\D{}^s}{\D{u^s}}\exp\brackets{-u^2} $
	and the corresponding weights $\cbrackets{w_i}_{i=1}^s$ are given by $
	w_i = \frac{2^{s-1} s!\sqrt{\pi}}{s^2 H_{s-1}\brackets{u_i}^2} $. 
	\label{thm:gauss-hermite}
\end{theorem}
Such approximations have been studied theoretically, with convergence rates provided for polynomials and functions of limited regularity. More specifically, the Gauss--Hermite approximation of order $s$ is exact and, thus, optimal for all polynomials of degree $2s-1$ or less~\cite{AbramowitzStegun}. For functions $h \in \mathcal{C}^{2s}$, the error of the Gauss--Hermite quadrature is given by~\cite{Stoer02}
\begin{align}
\begin{aligned}
& \mathcal{E}_s\brackets{h} = & \\
&\int_{-\infty}^{\infty} h\brackets{u}\exp\brackets{-u^2}\D{u} - \sum_{i=1}^{s}w_ih\brackets{u_i} = & \\
&  \frac{s!  \cdot \sqrt{\pi} }{2^s\cdot (2s)!} h^{(2s)}\brackets{\hat{u}} \ , &
\end{aligned}
\end{align}
where $\hat{u}\in\brackets{-\infty, \infty}$. Xiang and Bornemann~\cite{Xiang12} have studied convergence rates of the Gaussian quadrature for functions of limited regularity. The regularity of an integrand is expressed via the decay rate of its expansion coefficients in the basis formed by the Chebyshev polynomials of the first kind. In particular, if the expansion coefficients $a_i \in \mathcal{O}(i^{-p-1})$ for some $p>0$ (where $a_i$ corresponds to the Chebyshev polynomial of the $i$-th degree) then the error of the quadrature approximation of order $s$ can be upper bounded by $\mathcal{O}(s^{-p-1})$ for $p>2$. For $0<p<2$, on the other hand, the guaranteed convergence rate is sightly slower and can be upper bounded by $\mathcal{O}(s^{-\nicefrac{3p}{2}})$. These results can provide theoretically well founded guidelines for selecting the approximation order and quantify the trade-offs between approximation quality and computational costs.

\subsection{Illustration of the Approximation Scheme with Two Priors}
\label{subsec:illustration-priors}

In~\cite{kingma15}, it has been argued that \emph{log-scale uniform} priors provide a theoretical justification for the dropout regularization technique~\cite{SrivastavaHKSS14} frequently used in the training of neural networks.
The Bayesian aspect of that justification has recently been disputed in~\cite{hron18a} but the technique can still be viewed as performing penalized log-likelihood estimation with the Kullback--Leibler divergence term acting as regularizer.
The prior is given by
$p_{r, \mathrm{lsu}} \brackets{\log \absolute{\Delta_i}} \propto \mathrm{const},$ or equivalently  $p_{r, \mathrm{lsu}} \brackets{\absolute{\Delta_i}} \propto \nicefrac{1}{\absolute{\Delta_i}}$, where $\Delta_i$ is some network parameter. 
Two different second order approximations of the Kullback--Leibler divergence between Gaussian mean field posteriors and this prior distribution were provided in~\cite{kingma15} and~\cite{molchanov17a}.  
We propose an alternative Gauss--Hermite approximation, formalized in the following proposition. Just as in~\cite{SrivastavaHKSS14} and~\cite{kingma15}, we employ a parametrization of variational Gaussian mean field known as the \emph{dropout posterior}, with mean parameter $\mu_j$ and variance $\sigma_j^2= \alpha_j \mu_j^2$ specified via a scaling parameter $\alpha_j > 0$ (for all $1 \leq j \leq p$). 
\begin{restatable}{proposition}{propLogScaleUniform}
	\label{prop:log-scale-unif}
	The \acro{kl} divergence between a Gaussian distribution with the dropout parametrization of variance and a log-scale uniform prior can be approximated by
	\begin{align*}
	\kl \brackets{q \mid \mid p_{r, \mathrm{lsu}}} \approx - \nicefrac{1}{2} \log \alpha + \nicefrac{1}{\sqrt{\pi}} \sum_{i=1}^s w_i \log \absolute{v_i} + \mathrm{const.} \ ,
	\end{align*}
	where $v_i = \sqrt{2\alpha}u_i + 1$ (for all $1\leq i \leq s$)
	and the $\cbrackets{u_i}_{i=1}^s$ are roots of the Hermite polynomial with corresponding quadrature weights $\cbrackets{w_i}_{i=1}^s$.
\end{restatable}

\begin{proof}
	From~\cite[][Appendix C]{kingma15}, we know that the Kullback--Leibler divergence term is given by
	\begin{align*}
	\kl \brackets{q \mid \mid p_{r,\mathrm{lsu}}}= \mathbb{E}_{\mathcal{N}\brackets{\epsilon \mid 1, \alpha}}  \SBrackets{\log \absolute{\epsilon}}- \frac{1}{2} \log \alpha + \mathrm{const.} 
	\end{align*}
	The expectation with respect to the Gaussian random variable $\epsilon$ can be re-written as
	\begin{align*}
	\begin{aligned}
	& \mathbb{E}_{\mathcal{N}\brackets{\epsilon \mid 1, \alpha}}  \SBrackets{\log \absolute{\epsilon}} = & \\
	&  \frac{1}{\sqrt{2\pi\alpha}} \int \exp \brackets{-\frac{\brackets{\epsilon - 1}^2}{2\alpha}} \log \absolute{\epsilon} \D{\epsilon}=  &\\
	& \frac{1}{\sqrt{\pi}} \int \log \absolute{\sqrt{2\alpha}t+1} \exp\brackets{-t^2} \D{t} \ . &
	\end{aligned}
	\end{align*}
	The result now follows from Theorem~\ref{thm:gauss-hermite} by taking $h\brackets{t}= \log \absolute{\sqrt{2\alpha}t+1}$.
\end{proof}

The scale-mixture is another prior distribution frequently used in variational inference, first proposed in~\cite{blundell15}. It resembles the so called spike and slab prior~\cite{Chipman96,McCulloch93,Mitchell88} and is given by
\begin{align*}
\label{eq:spike-and-slab}
\begin{aligned}
& p_{r, \mathrm{sm}}\brackets{\Delta_i \mid \xi,\eta_1,\eta_2,\lambda} = & \\ 
&\lambda \cdot \mathcal{N}\brackets{\Delta_i \mid \xi, \eta_1^2} +
  (1 - \lambda) \cdot  \mathcal{N}\brackets{\Delta_i \mid \xi, \eta_2^2} \ , &
\end{aligned}
\end{align*}
where $\Delta_i$ is a parameter of the model (see Eq.~\ref{eq:svi}), $\eta_1^2$ and $\eta_2^2$ are prior (variance) hyper-parameters with $\eta_1 \ll \eta_2$, $\xi$ is the prior mean, and $0 \leq \lambda \leq 1$ is the mixture scale. The hyper-parameters of the prior distributions (i.e., $\eta_1$, $\eta_2$, $\lambda$, and $\xi$) are kept fixed during optimization and can be chosen via cross-validation. The first mixture component is chosen such that $\eta_1 \ll 1$, which forces many of the variational parameters to concentrate tightly around the prior mean $\xi$ (e.g., around zero for $\xi=0$). The second mixture component has higher variance and heavier tails allowing parameters to move further away from the mean. 
The prior variance hyper-parameters are shared between all the network parameters and this is an important difference compared to approaches based on the spike and slab prior~\cite{Mitchell88,McCulloch93,Chipman96}, where each model parameter has a different prior variance. The following proposition provides means for approximating the divergence term between a Gaussian mean field variational distribution and this prior function.

\begin{restatable}{proposition}{propScaleMixture}
	\label{prop:scale-mixture}
	The \acro{kl} divergence between a Gaussian distribution with the dropout parametrization of variance and a scale-mixture prior can be approximated by
	\begin{align*}
	\begin{aligned}
	& 
	\kl \brackets{q \mid \mid p_{r, \mathrm{sm}}} \approx &\\
	&
	- \log \sqrt{2\pi \alpha\mu^2} - \nicefrac{1}{\sqrt{\pi}}\sum_{i=1}^{s} w_i \log p_{r, \mathrm{sm}} \brackets{v_i} - \nicefrac{1}{2} \ ,  &
	\end{aligned}
	\end{align*}
	where $v_i=\brackets{\sqrt{2\alpha}u_i+1}\mu$ and the
	 $\cbrackets{u_i}_{i=1}^s$ are roots of the Hermite polynomial with corresponding quadrature weights $\cbrackets{w_i}_{i=1}^s$, $\alpha$ and $\mu$ are variational parameters, and $p_{r, \mathrm{sm}}$ is some scale-mixture prior distribution.
\end{restatable}

\begin{proof}
	We can re-write the divergence term as
	\begin{align*}
	\begin{aligned}
	& \kl \brackets{q \mid \mid p_{r,\mathrm{sm}}}=  & \\
	&\int q\brackets{u} \log q\brackets{u} \D{u} - \int q\brackets{u} \log p_{r,\mathrm{sm}}\brackets{u} \D{u} = & \\
	& -H(q) - \mathbb{E}_q\sbrackets{\log p_{r,\mathrm{sm}}\brackets{u}}  , &
	\end{aligned}
	\end{align*}
	where $H\brackets{q}$ denotes the entropy of the univariate Gaussian distribution given by
	\begin{align*}
	q\brackets{u}=\frac{1}{\sqrt{2\pi\alpha\mu^2}} \exp \brackets{-\frac{\brackets{u-\mu}^2}{2\alpha\mu^2}} \ .
	\end{align*}
	As the entropy of a Gaussian distribution defines an analytically tractable integral~\citep[e.g., see][]{Kullback59,Rasmussen05}, we have that the entropy of $q$ is given by 
	\begin{align*}
	H\brackets{q}=\log \sqrt{2\pi \alpha \mu^2}+\nicefrac{1}{2} \ .
	\end{align*}
	On the other hand, the expected log-likelihood of the scale-mixture prior can be approximated using the Gauss-Hermite quadrature by observing that
	\begin{align*}
	\begin{aligned}
	& \mathbb{E}_q\sbrackets{\log p_{r,\mathrm{sm}}\brackets{u}} = & \\ & \frac{1}{\sqrt{2\pi\alpha\mu^2}} \int \exp \brackets{-\frac{\brackets{u-\mu}^2}{2\alpha\mu^2}} \log p_{r,\mathrm{sm}} \brackets{u} \D{u}= & \\
	& \frac{1}{\sqrt{\pi}} \int  \log p_{r,\mathrm{sm}}\brackets{\sqrt{2\alpha\mu^2}t+\mu} \exp\brackets{-t^2} \D{t} \ . &
	\end{aligned}
	\end{align*}
	The result now follows from Theorem~\ref{thm:gauss-hermite} by taking $h\brackets{t}=\log p_{r,\mathrm{sm}}\brackets{\sqrt{2\alpha\mu^2}t+\mu}$.
\end{proof}

\section{Related Work}
\label{sec:rw}

An alternative to learning a discriminative model with non-adaptive features is to learn these features automatically as part of a neural architecture that takes raw speech as input. In addition to having a more flexible inductive bias such a model would be less susceptible to the information loss that is inherent to waveform compression by means of a projection to a lower dimensional feature space~\cite{Ager11,Ager2015PhonemeCI}. In particular, a model operating directly in the waveform domain has the potential to exploit local correlations within the signal that are typically discarded when computing Mel-filter bank values~\cite{hoshen2015}, as well as the information contained in a sequence of waveform samples without interruptions by frame boundaries characteristic to spectrograms and non-adaptive feature extraction techniques based on frame-based discrete Fourier transforms~\cite{tuske2014}. 
As a result of the latter, phonetic events on the boundaries of short frames are typically poorly described by filterbank features.

Whilst speech production embeds redundancies relevant for robustness, there are several challenges when dealing with these highly correlated raw speech inputs. In particular, the high dimensionality of waveform signals typically requires a larger number of parameters compared to standard features and a prolonged training time. Another difficulty is the fact that raw speech is known to be characterized by a large number of variations such as temporal distortion and speaker variability~\cite{sainath2015,Ghahremani2016}. Acoustic models based on neural networks operating directly in the waveform domain are, thus, likely to over-fit on small and moderately sized datasets without appropriate inductive bias. 
In this sense our approach, which combines variational inference with Lipschitz continuity of the operator mapping, provides a theoretical underpinning for the design and learning of effective waveform-based acoustic models. Previous work has also resorted to similar techniques for maintaining the balance between dataset size and model complexity. Watanabe et al.~\cite{watanabe03,watanabe04} have used variational inference for clustering of states in triphone hidden Markov models (\acro{hmm}) and learning the appropriate number of components in Gaussian mixture models (\acro{gmm}). In contrast to this, we use variational inference to learn a stochastic convolutional network that models the conditional probability of a triphone state-id given an input waveform frame. 

Graves~\cite{Graves} and Braun and Liu~\cite{braun} have used variational inference to learn a recurrent neural network as part of an end-to-end acoustic model. While the latter approach does not have an explicit \textsc{kl} divergence term characteristic to variational inference, there is a sparsity inducing penalty over the parameters defining standard deviations, which under a suitable prior could be seen as an instance of \textsc{kl} divergence. In both of these works it was observed that parameter uncertainty is correlated with the importance of individual parameters for the speech recognition tasks considered. Similarly, Hu et al.~{\cite{bnngp}} have proposed a Bayesian neural network that allows for learning with more expressive activation functions in the context of multi-layer perceptrons and standard recurrent neural networks. In particular, each hidden layer of the model relies on Bayesian averaging relative to a weight prior when computing the corresponding outputs, and variational inference for dealing with the resulting analytically intractable integrals. A Bayesian approach coupled with variational inference has also been used in~{\cite{blhuc}} for speaker adaptation. The main difference to this line of work is that neither of those models operates in the waveform domain, but rely on low-dimensional feature spaces generated by \acro{fbank} or \acro{mfcc} features. This allows for scalable inference of recurrent models, which is known to be computationally expensive for high dimensional inputs such as waveform signals. Moreover, prior work in speech recognition (to the best of our knowledge) considers variational inference independently of Lipschitz continuity and other design principles that could allow for learning of robust models in small scale settings.
Recently, an approach for modulation filter-learning based on an encoder-decoder architecture and variational inference has been considered in~\cite{Agrawal19a} and~\cite{Agrawal19b}. The encoder takes as input a Mel-spectrogram constructed using speech segments of fixed length and learns its latent representation. 
The optimization of encoder-decoder parameters is performed using variational inference and the learned filters are then used to generate features that are used as input to an \acro{mlp}. In contrast to this, we use variational inference to learn filters jointly with other network parameters (i.e., filterbank-based feature extraction/learning is not done independently of training other network modules).

A common characteristic of previous approaches for waveform-based speech recognition is the use of relatively large datasets~\cite{sainath2015,zhu2016}. In such a regime, waveform-based acoustic models are competitive with architectures relying on standard features (i.e., \acro{mfcc}, \acro{fbank}, and \acro{fmllr}). Another difference compared to our approach is that previous architectures typically employ a convolutional layer with weighted $\ell_1$ or $\ell_2$ pooling ($25$ ms long frames) to emulate filterbank features and reduce the dimension of the representation quickly~\cite{hoshen2015,palaz2013}. In contrast to this, we perform gradual compression of the waveform sub-band decomposition via max pooling and thus overcome the information loss inherent in standard features. Moreover, we use the \acro{relu} non-linearity throughout the network and do not apply the \acro{log} operator to the outputs of the initial block. Sainath et al.~\cite{sainath2015} propose an architecture that takes raw speech inputs and applies one-dimensional convolutions first in the time-domain and then the frequency-domain, designed to extract band-pass features from the waveform.
The architecture itself is a recurrent net that requires more than $2,000$ hours of training data to match the performance of models with standard features. 
Similarly, Zhu et al.~\cite{zhu2016} combine two convolutional layers with recurrent blocks in end-to-end training, requiring more than $2,400$ hours of training data for state-of-the-art results. Ghahremani et al.~\cite{Ghahremani2016} proposed a feedforward architecture based on a convolutional feature extraction layer, with the outputs of that block passed to a deep time-delay neural network (\acro{tdnn}). 
The empirical results indicate that the approach is competitive with \acro{mfcc}-based architectures on large datasets. It has not been evaluated on noisy speech and it is unclear how well it would generalize from small datasets.

Our architecture performs parametric sub-band decomposition of speech waveforms and it is most closely related to \acro{sincnet}~\cite{sincnet}, which employs three \acro{1d} convolutional layers on top of the parametric block. \acro{sincnet} is considered to be the state-of-the art model for waveform-based speech recognition.
A related architecture is \acro{sinc}$^2$\acro{net}; this links a parametric convolution block to an \acro{mlp}~\cite{edinburgh}. Recently, complex-valued parametric filters have been used to initialize a complex non-parametric convolution block in a deep network for end-to-end speech recognition~\cite{ZeghidourUKSSD18,Zeghidour18a,Zeghidour18b}. In comparison to~\cite{ZeghidourUKSSD18}, we show that our approach generalizes better on the small \acro{timit} dataset. In our experiments, we use the \acro{sincnet} architecture (code available) as a representative baselines from this class.

Recently, an approach based on concatenation of multiple convolutional blocks was proposed~{\cite{multispan}}, in which convolutional blocks capture different contexts in time and learn band-pass filters that are more expressive than classic Mel-filterbanks, which operate on a single fixed context. The approach was evaluated on both noisy and conversational speech. In our experiments, we compare to this baseline and demonstrate statistically significant improvement on the \textsc{ami-ihm} dataset ($12\%$ relative).

\begin{table*}[!htb]
	\centering\fontsize{8}{10}\selectfont  
	\caption{The table reports the average phoneme error rates (standard deviations are provided in the brackets), obtained using variational \acro{parznets 1d} and Gaussian mean field (variational) inference on the \acro{timit} dataset.}
	\begin{tabular}[t]{l|c|c|c|c|c|c}
		\multirow{5}{*}{\fontsize{6}{8}\selectfont  \textsc{sample}} & \multicolumn{4}{c|}{\fontsize{6}{8}\selectfont \textsc{vi -- log-scale uniform}} & \multicolumn{2}{c}{\fontsize{6}{8}\selectfont \textsc{vi -- scale mixture}} \\\cline{2-7}
		& \multicolumn{3}{c|}{\fontsize{6}{8}\selectfont \textsc{squared epanechnikov}} & \textsc{\fontsize{6}{8}\selectfont gauss} & \multicolumn{2}{c}{\textsc{\fontsize{6}{8}\selectfont squared epanechnikov}} \\\cline{2-7}
		& \textsc{\thead[l]{\fontsize{6}{8}\selectfont non-adaptive mel-filters\\\hline \fontsize{6}{8}\selectfont kl approximation:\\ \fontsize{6}{8}\selectfont hermite-gauss quad.}} & \textsc{\thead[l]{\fontsize{6}{8}\selectfont adaptive filters\\\hline\fontsize{6}{8}\selectfont kl approximation:\\\fontsize{6}{8}\selectfont hermite-gauss quad.}} & \textsc{\thead[l]{\fontsize{6}{8}\selectfont adaptive filters\\\hline\fontsize{6}{8}\selectfont kl approximation:\\\fontsize{6}{8}\selectfont molchanov~et~al~\cite{molchanov17a}}} &  \textsc{\thead[l]{\fontsize{6}{8}\selectfont adaptive filters\\\hline\fontsize{6}{8}\selectfont kl approximation:\\\fontsize{6}{8}\selectfont hermite-gauss quad.}} & \textsc{\thead[l]{\fontsize{6}{8}\selectfont adaptive filters\\\hline\fontsize{6}{8}\selectfont kl approximation:\\\fontsize{6}{8}\selectfont hermite-gauss quad.}} & \textsc{\thead[l]{\fontsize{6}{8}\selectfont adaptive filters\\\hline\fontsize{6}{8}\selectfont kl approximation:\\\fontsize{6}{8}\selectfont mcmc~\cite{blundell15}}} \\
		\hline
		
		\hline
		\textsc{dev} & $15.02$ ($\pm0.26$) & $14.95$ ($\pm0.14$) & $\textbf{14.77}\ $ ($\pm0.15$) & $14.83$ ($\pm0.13$)  & $15.64$ ($\pm0.11$) & $15.58$ ($\pm0.20$)\\
		\hline
		\textsc{test} & $16.95$ ($\pm0.25$) & $\textbf{16.52}\ $ ($\pm0.22$) & $16.63$ ($\pm0.23$) & $16.60$ ($\pm0.22$) & $17.41$ ($\pm0.17$) & $17.56$ ($\pm0.16$)\\
		
		\hline
		
		\hline
		
		\hline
	\end{tabular}
	\label{tbl:timit-1}
\end{table*}

\section{Experiments}
\label{sec:experiments}

We evaluate the proposed approach with a series of experiments on three different datasets: \acro{timit}~\cite{timit}, \acro{aurora4}~\cite{aurora4}, and \textsc{ami-ihm}~{\cite{renals2007recognition}}. In all the experiments\footnote{A detailed setup of our experiments along with the source code can be found in the project repository~\url{https://bitbucket.org/doglic/asr/}.}, we train a context dependent hybrid \acro{hmm} model based on frame labels (i.e., \acro{hmm} state ids) generated using a triphone model from Kaldi~\cite{kaldi} with $25$ ms frames and $10$ ms stride between the successive frames. The data splits (train/validation/test) originate from the Kaldi framework. In the pre-processing step, we assign the Kaldi frame label to the $200$ ms long segment of raw speech centered at an original Kaldi frame (keeping $10$ ms stride between the successive frames of raw speech). To be consistent with our baselines on \acro{timit}, we generate frame labels using the \acro{dnn}  
triphone model and decoding configuration from~\cite{sincnet}. For \acro{aurora4}, on the other hand, we generate frame labels using both \acro{gmm} and \acro{dnn} triphone models, relying on the default decoder configuration from Kaldi.

We describe below four sets of experiments. The first aims at demonstrating the impact of particular design choices on the effectiveness of acoustic models. More specifically, our empirical results show that: modulation filter learning can improve the performance of acoustic models in a statistically significant way (subsection A, below), the proposed approximation scheme for the Kullback--Leibler divergence term is generally more effective than previous approaches (subsection B, below), modulation filter learning moves away from the initial solution and converges to different distributions of modulation frequencies for different learning tasks (subsection C, below),
and probabilistic parametrization of the neural architecture contributes to a $7.4\%$ relative improvement in the error rate compared to the deterministic one (subsection D, below). The second set of experiments (subsections D and E, below) is aimed at showing that the proposed approach does not over-fit on what is considered to be a small dataset in speech recognition (i.e.\ \acro{timit}). Moreover, the results also indicate that a combination of variational inference and Lipschitz continuous architectures for waveform-based speech recognition such as \acro{parznets} does not require large training datasets to outperform models based on standard filterbank features. The third experiment (subsection E, below)
deals with noisy speech and shows that the proposed approach can learn an effective noise robust representation of waveform signals. The fourth and final experiment aims at demonstrating the effectiveness of the proposed approach on conversational speech (i.e., \textsc{ami-ihm}), with approximately $80$ hours of audio. The experiment shows a clear improvement over recently proposed waveform-based approaches ($12\%$ relative) and a competitive performance relative to filterbank architectures known for their effectiveness on this dataset. We also observe that variational inference consistently contributes to an improvement in the error rate compared to the deterministic models.

\subsection{Can modulation filter learning improve the effectiveness of waveform-based acoustic models?}
\label{subsec:exp-filter-learning}

The goal of this experiment is to demonstrate that filter optimization can be more effective than non-adaptive filtering of speech signals, in a way that is statistically significant. To that end, we train two neural networks with identical architectures (see Fig.~\ref{fig:arch}) using variational inference with the Kullback--Leibler divergence term approximated via the Hermite--Gauss (\acro{hg}) quadrature: \emph{i}) a neural network with non-adaptive Parzen filters initialized just as in Mel-frequency coefficients (denoted with \acro{mel-filters} in Tables~\ref{tbl:timit-1} and~\ref{tbl:aurora4}), and \emph{ii}) the joint filter and neural network learning proposed in this work (see \acro{adaptive filters, hermite-gauss quad.} under log-scale uniform prior \acro{vi} in Tables~\ref{tbl:timit-1} and~\ref{tbl:aurora4}). The Parzen filters of the latter adaptive operator are initialized exactly as the non-adaptive ones. To assess whether one method performs statistically significantly better than the other on \acro{timit}, we perform the paired Welch t-test~\cite{Welch47} based on $5$ repetitions of the experiment. The t-test indicates that filter learning is with $90\%$ confidence statistically significantly better than non-adaptive filtering. We similarly studied performance on \acro{aurora4}, which is a much larger dataset than \acro{timit} where repeated training is time consuming and expensive. However, the dataset contains $14$ different test samples and this allows us to employ the Wilcoxon signed rank test~\cite{Wilcoxon45,Demsar06} to again establish whether one approach is statistically significantly better than the other. The test indicates that filter learning is with $95\%$ confidence statistically significantly better than non-adaptive filtering on \acro{aurora4} (see e.g.\ Table~\ref{tbl:aurora4}).

\begin{table}[!t]
	\centering\fontsize{6}{8}\selectfont  
	\caption{\acro{aurora4}, word error rates obtained using different test samples.}
	\begin{tabular}[t]{l|c|c|c|c|c|c}
		 & \multicolumn{4}{c|}{\fontsize{5}{7}\selectfont  \textsc{vi -- log-scale uniform} } &  \multicolumn{2}{c}{\fontsize{5}{7}\selectfont  \textsc{vi -- scale mixture} }  \\\cline{2-7}
		& \multicolumn{3}{c|}{\textsc{\fontsize{5}{7}\selectfont  8 x cnn}} & \textsc{\fontsize{5}{7}\selectfont  10 x cnn} & \multicolumn{2}{c}{\textsc{\fontsize{5}{7}\selectfont  8 x cnn}} \\\cline{2-7}
		\textsc{\fontsize{5}{7}\selectfont adaptive filters} &   & $\checkmark$ &$\checkmark$& $\checkmark$ &$\checkmark$ &  $\checkmark$ \\
		\textsc{\fontsize{5}{7}\selectfont kl: hermite-gauss} & $\checkmark$ & $\checkmark$ &  & $\checkmark$ & $\checkmark$ &  \\
		\textsc{\fontsize{5}{7}\selectfont kl: molchanov~et~al.} &  &  & $\checkmark$ & &  &   \\
		\textsc{\fontsize{5}{7}\selectfont kl: mcmc} &  &  &  & &  & $\checkmark$  \\
		\hline
		
		\hline
		
		\multicolumn{6}{l}{\textsc{a. same microphone}}\\
		\hline
		\textsc{\fontsize{5}{7}\selectfont clean (a)} & $3.05$ & $2.88$ & $2.84$ & $2.78$ & $3.12$ & $\textbf{2.71}$  \\
		\hline
		\multicolumn{6}{l}{\textsc{b. same microphone}}  \\
		\hline
		\textsc{\fontsize{5}{7}\selectfont car} & $3.29$ & $3.34$ & $3.14$ &  $\textbf{3.10}$ & $3.29$ & $3.25$ \\
		\textsc{\fontsize{5}{7}\selectfont babble}& $4.63$ & $4.33$ & $4.84$ & $\textbf{4.26}$ & $4.54$ & $4.84$  \\
		\textsc{\fontsize{5}{7}\selectfont restaurant}& $6.46$ & $\textbf{6.00}$ & $6.18$ & $6.54$ & $6.65$ & $6.37$  \\
		\textsc{\fontsize{5}{7}\selectfont street}& $5.87$ & $5.87$ & $5.88$ & $\textbf{5.70}$ & $6.22$ & $6.16$  \\
		\textsc{\fontsize{5}{7}\selectfont airport}& $4.76$ & $4.45$ & $4.58$ & $\textbf{4.43}$ & $4.78$ & $4.61$  \\
		\textsc{\fontsize{5}{7}\selectfont train}& $6.41$ & $6.33$ & $\textbf{6.30}$ & $6.35$  & $\textbf{6.30}$ & $ 6.35$\\
		\hline
		\textsc{\fontsize{5}{7}\selectfont average (b)} & $5.24$ & $\textbf{5.05}$ & $5.15$ & $5.06$ & $5.30$ & $5.26$  \\
		\hline
		\multicolumn{6}{l}{\textsc{c. different microphones}}\\
		\hline
		\textsc{\fontsize{5}{7}\selectfont clean (c)}& $5.90$ & $5.59$ & $6.02$ & $\textbf{5.27}$ & $6.09 $& $5.96$  \\
		\hline
		\multicolumn{6}{l}{\textsc{d. different microphones}}\\
		\hline
		\textsc{\fontsize{5}{7}\selectfont car}& $9.79$ & $9.30$ & $9.36$ & $\textbf{9.10}$ & $9.84$ & $10.14$  \\
		\textsc{\fontsize{5}{7}\selectfont babble}& $15.84$ & $15.41$ & $16.01$ & $\textbf{14.78}$  & $16.07$& $16.16$ \\
		\textsc{\fontsize{5}{7}\selectfont restaurant}& $20.08$ & $20.77$ & $21.39$ & $\textbf{19.56}$ & $21.15$ & $21.24$  \\
		\textsc{\fontsize{5}{7}\selectfont street}& $17.31$ & $\textbf{16.80}$ & $17.71$ &  $17.28$ & $17.65$ & $18.61$ \\
		\textsc{airport}& $14.70$ & $13.88$ & $14.65$ & $\textbf{13.30}$ & $14.70$ & $14.94$  \\
		\textsc{\fontsize{5}{7}\selectfont train}& $17.43$ & $\textbf{16.99}$ & $17.49$ & $17.07$ & $17.64$ &  $17.90$  \\
		\hline
		\textsc{\fontsize{5}{7}\selectfont average (d) }& $15.86$ & $15.53$ & $16.10$ & $\textbf{15.18}$ & $16.18$ & $16.50$  \\
		\hline
		
		\hline
		\textsc{\fontsize{5}{7}\selectfont average (all)}& $9.68$ & $9.42$ & $9.74$ & $\textbf{9.25}$ & $9.86$ & $9.95$  \\
		\hline
		
		\hline
	\end{tabular}
	\label{tbl:aurora4}
\end{table}

\subsection{How effective is the Gauss--Hermite approximation scheme?}

Having established that modulation filter learning can be significantly better than static filtering, we proceed to show that Hermite--Gauss quadrature is an effective scheme for the approximation of the Kullback--Leibler divergence term acting as a regularizer in variational inference. In particular, we compare the effectiveness of neural networks learned via variational inference and existing strategies for approximation of the Kullback--Leibler divergence term, defined using the log-scale uniform~\cite{molchanov17a} and scale mixture priors~\cite{blundell15}. Table~\ref{tbl:timit-1} (see \acro{squared epanechnikov} modulation filters, \acro{test} sample) provides the results on \acro{timit} and shows that the approximation based on the Hermite--Gauss quadrature (see \acro{hermite-gauss quad.} columns) is on average better than existing approximation schemes (see \acro{molchanov et al.} and \acro{mcmc} columns). However, the Welch t-test does not show a statistically significant improvement of the Hermite--Gauss quadrature over the alternatives on this dataset. Table~\ref{tbl:aurora4} summarizes our results on \acro{aurora4} and demonstrates a significant improvement over the baselines when using the Hermite--Gauss quadrature to approximate the Kullback--Leibler divergence term. More specifically, the Wilcoxon signed rank test in the case of log-scale uniform prior shows that the approximation based on the Hermite--Gauss quadrature is with $95\%$ confidence statistically significantly better than the state-of-the-art approximation proposed in~\cite{molchanov17a}.

\subsection{Do modulation frequencies move away from the initial solution and converge to different distributions for different learning tasks?}

The goal of this experiment is to demonstrate that the optimization of modulation filters changes the initial distribution of modulation frequencies and bandwidths. Fig.~\ref{fit:filters} provides a comparison of kernel density estimators for modulation frequencies and filter bandwidths. From the figure, it is evident that the initial and optimized distributions are quite different for filter bandwidths on both datasets. Moreover, there is an interesting difference between the distributions of modulation frequencies between \acro{timit} and \acro{aurora4} datasets, which might be due to multi-condition training and various noise conditions characteristic to \acro{aurora4}.

\begin{figure}[!htb]
	\centering
	\input{filters.tikz}
	\caption{Comparison of the initial distributions of modulation frequencies and bandwidths to those at the end of the training process.}
	\label{fit:filters}
\end{figure}

	\begin{table}[!htb]
	\centering\fontsize{8}{10}\selectfont  
		\caption{Comparison of phoneme error rates obtained in our experiments on \acro{timit} to the ones reported for relevant feedforward nets.}
	\begin{tabular}[t]{l|c|c}
		\textsc{method} & \textsc{avg} & \textsc{min}\\
		\hline
		
		\hline
		\multicolumn{3}{l}{\textsc{a. raw speech baselines (optimized filters)}}\\
		\hline
		\textsc{variational parznets} &  $\textbf{16.5}$ & $\textbf{16.2}$ \\
		\textsc{deterministic parznets} & $17.7$ & $17.5$\\
		\textsc{sincnet}~\cite{sincnet,ravanelli2018pytorchkaldi} & $17.5$ & $17.2$\\
		\textsc{sinc}$^2$\textsc{net}~\citep{edinburgh} & $-$ & $16.9$\\
		\textsc{end-to-end cnn}~\cite{ZeghidourUKSSD18}  & $-$ &$18.0$\\
		\textsc{raw speech cnn}~\cite{sincnet}  & $18.3$ &$18.1$\\
		\hline
		\multicolumn{3}{l}{\textsc{b. standard features (non-adaptive filters)}}\\
		\hline
		\textsc{fmllr + mlp} & $16.9$ & $16.7$\\
		\textsc{mfcc + mlp}~\cite{Ravanelli2016BatchnormalizedJT} & $18.1$ & $17.8$\\
		\textsc{multi-res dss + cnn \& mlp}~\cite{dssnn} & $-$& $17.4$ \\
		\hline
		
		\hline
	\end{tabular}
	\label{tbl-timit-ref}\vspace{-0.5ex}
	\end{table}

\subsection{How does the approach fare relative to state-of-the-art feedforward models on \textsc{timit}?}
\label{subsec:exp-timit}

Table~\ref{tbl-timit-ref} summarizes our empirical results in comparison to state-of-the-art feedforward architectures on \acro{timit}.  
In addition to the lowest obtained error rate (denoted with \acro{min}), we also report the average result over $5$ simulations. A comparison to previously reported results for waveform-based speech recognition indicates that our approach performs the best on average on this task. Moreover, this is the first such approach that outperforms all the feedforward architectures built on top of standard non-adaptive features. 
Our results also show that variational inference contributes to a $7.4\%$ relative improvement on this dataset over a 
deterministic network with identical architecture (see \acro{deterministic parznets} in Table~\ref{tbl-timit-ref}).
We note here that recent work has reported lower error rates on \acro{timit} using recurrent nets and statically extracted features. In particular,~\cite{pytorchkaldi} reports the following error rates for gated recurrent units (\acro{gru}): \acro{li-gru} $15.8\%$ and \acro{li-gru fmllr} $14.8\%$. In the waveform domain with low-resources (i.e., small datasets such as \acro{timit}) recurrent nets perform worse than feedforward models. In particular, our best result on this dataset with recurrent nets in the waveform domain was $18.8\%$, which is significantly worse than the best observed result with \acro{parznets} (i.e., $16.2\%$). The good performance of models based on \acro{fmllr} features should not come as a surprise, because that feature extraction technique performs speaker and domain adaptation as well.  
Our future work will explore recurrent architectures in the waveform-domain, combined with regularization mechanisms provided by variational inference.

	\begin{table}[!t]
	\centering\fontsize{7}{9}\selectfont  
	\caption{Word error rates obtained on \acro{aurora4} using multi-condition training and input/context frames of $200$ ms (\textsc{a}: clean speech with same microphone, \textsc{b}: noisy speech with same microphone, \textsc{c}: clean speech with different microphones, \textsc{d}: noisy speech with different microphones).}
	\begin{tabular}[t]{l|c|c|c|c|c}
		\textsc{method} & \textsc{a} & \textsc{b} & \textsc{c} & \textsc{d} & \textsc{avg}\\
		\hline
		
		\hline
		\multicolumn{6}{l}{\textsc{a. raw speech \& var. baselines (optimized filters)}}\\
		\hline
		\multicolumn{6}{l}{$\qquad$\textsc{dnn alignments}}\\
		\hline
		\textsc{var. parznets (10 x cnn1d)} &  $2.22$ & $4.50$ & $4.71$ & $14.72$ & $\textbf{8.73}$ \\
		\textsc{det. parznets (10 x cnn1d)} & $2.35$ & $4.73$ & $4.86$ & $15.48$ & $9.17$\\
		\textsc{var. parznets (8 x cnn1d)} &  $2.15$ & $4.50$ & $5.28$ & $15.07$ & $8.92$ \\
		\textsc{det. parznets (8 x cnn1d)} & $2.24$ & $4.61$ & $5.75$ & $15.48$ & $9.18$\\
		\hline
		\multicolumn{6}{l}{$\qquad$\textsc{gmm alignments}}\\
		\hline
		\textsc{var. parznets (10 x cnn1d)} &  $2.78$ & $5.06$ & $5.27$ & $15.18$ & $9.25$ \\
		\textsc{var. parznets (8 x cnn1d)} &  $2.88$ & $5.05$ & $5.59$ & $15.53$ & $9.42$ \\
		\textsc{sincnet}~\cite{ravanelli2018pytorchkaldi} & $3.42$ & $6.33$ & $6.13$ & $16.99$ & $10.68$\\
		\textsc{cvae feats + mlp}~\cite{Agrawal19a,Agrawal19b} & $3.50$ & $7.40$ & $6.90$ & $17.10$ & $11.20$\\
		\hline
		\multicolumn{6}{l}{\textsc{b. standard features (non-adaptive filters)}}\\
		\hline
		\textsc{fbank + vd10 x cnn2d}~\cite{vdcnn} & $4.13$ & $6.62$ & $5.92$ & $14.53$ & $9.78$ \\
		\textsc{fbank + vd8 x cnn2d}~\cite{vdcnn} & $3.72$ & $6.57$ & $5.83$ & $14.79$ & $9.84$ \\
		\textsc{fmllr + mlp} & $3.34$ & $6.27$ & $5.74$ & $16.04$ & $10.21$ \\
		\textsc{mfcc + mlp} & $4.28$ & $7.44$ & $8.73$ & $18.71$ & $12.14$\\
		\textsc{dss (utt. norm.) + junct. net} & $3.05$ & $5.82$ & $6.11$ & $15.94$ & $9.98$ \\
		\textsc{dss (w/o norm.) + junct. net} & $4.09$ & $6.35$ & $8.24$ & $19.07$ & $11.78$ \\
		\hline
		
		\hline
	\end{tabular}
	\label{tbl-aurora4-ref}\vspace{-1ex}
\end{table} 

\subsection{How does the approach fare relative to state-of-the-art feedforward models on \textsc{aurora4}?}
\label{subsec:exp-aurora4}

\acro{aurora4} is a medium vocabulary task based on clean speech from the Wall Street Journal (\acro{wsj0}) corpus~\cite{wsj0}. The clean speech was corrupted by six different noise types at different \acro{snr}s. The test sets consist of noise corrupted utterances recorded by a primary and a secondary microphone. 
In Table~\ref{tbl-aurora4-ref} we provide a summary of our results on this dataset relative to state-of-the-art feedforward architectures. 
The first experiment compares our approach (\textsc{8 x cnn1d}) 
to the state-of-the-art architecture for waveform-based speech recognition~\mbox{\cite[][\acro{sincnet}]{sincnet}} and shows a statistically significant~\mbox{\cite[][Wilcoxon test, $95\%$ confidence]{Demsar06,Wilcoxon45}} improvement over that baseline. We also compare to a recent approach for modulation filter-learning using encoder-decoder architecture and variational inference~\mbox{\cite{Agrawal19a,Agrawal19b}}. The results again show (with $95\%$ confidence) that the proposed approach is statistically significantly better than the baseline from~\mbox{\cite{Agrawal19a,Agrawal19b}}. Following this, we compare our results to the error rates reported in~\cite{vdcnn} for $8$ and $10$-layer deep \acro{2d} convolutional networks (\acro{vdcnn2d}) based on statically extracted features using $200$ ms long raw-speech segments (i.e., $17$ \acro{fbank} frames). This might be an unfair comparison to our approach, because we use the less expressive \acro{1d} convolutions in our architecture. Still, the results indicate that the variational \acro{parznets} architecture with $8$ convolutional layers 
outperforms significantly the network with $10$ \acro{cnn2d} layers from~\cite{vdcnn}.  
Furthermore, we extend our architecture (Fig.~\ref{fig:arch}) to $10$ convolutional layers by employing time-padding in \acro{1d} convolutions to allow for another double convolutional block. The results indicate a further improvement in accuracy as a result of this modification.
 
Another particularly interesting observation is that the gains of our approach over noisy samples do not come as a result of performance degradation on clean speech. We note here that~\cite{vdcnn} reports a slightly better error rate with \acro{2d} convolutions and \acro{fbank} features when the context size is increased to $250$ ms (i.e.\ $21$ frames), in combination with time and frequency padding (\acro{wer} $8.81\%$). Table~\ref{tbl-aurora4-ref} (see \acro{dnn alignments}) shows that our approach provides a competitive error rate (\acro{wer} $8.73\%$) with smaller context size (i.e., $200$ ms) and less expressive time-padded \acro{1d} convolutions. Moreover, a recent approach based on multi-octave convolutions and $15$ such convolutional layers has achieved the error rate of $8.31\%$ on this dataset~\cite{moctcnn}. 

In a follow up work~{\cite{oglic20}}, we have investigated \textsc{parznets} with \textsc{2d} convolutional operators coupled with Bernoulli dropout layers (i.e.\ a special case of stochastic neural networks with variance parameter fixed over an entire network layer). This approach achieved a word error rate of $7.80\%$, which is the best reported number on this dataset for waveform-based speech recognition. Here, it is important to note that \textsc{1d parznets} baselines from~{\cite{oglic20}} employ time-padded convolutions and an extra fully connected layer in the \textsc{mlp} block compared to the neural architecture considered in this paper.

In addition to waveform-based baselines and deep convolutional networks operating with standard non-adaptive features, we have also compared our approach to a junction network~\cite{dssnn} coupled with first and second order deep scattering spectrum features (see Table~\ref{tbl-aurora4-ref}, \acro{dss + junc. net}). The latter is a non-adaptive wavelet-based feature extraction technique~\cite{mallat14} that generates features of different orders, with the first order coefficients approximately equal to \acro{mfcc}, and higher order coefficients recovering information lost at lower levels.  
Our experiments demonstrate that \acro{parznets} can outperform this approach, even when it is supplied with utterance level normalization. In parallel with this work, we have also proposed deep scattering power spectrum features~\cite{oglic20b}. The latter non-adaptive feature extraction technique coupled with the junction neural architecture and utterance level normalization performs on par with \acro{parznets} (\acro{wer} $8.83\%$). 
Given that deep scattering spectrum recovers information lost at lower levels, we hypothesize that this might be yet another indication for the relevance of information loss (characteristic to standard filterbank features) for robustness to standard noise corruptions.

\begin{table}[!t]
	\centering\fontsize{8}{10}\selectfont  
	\caption{The word error rates obtained on dev and eval sets of \textsc{ami-ihm} with various input features and neural architectures. We did not use any data augmentation techniques or \emph{i-vectors} in the experiments. Following the original Kaldi recipe, a \textsc{3-gram} language model built from the \acro{ami} and \acro{fisher} data was adopted. Some of the related baselines relied on a contextually more expressive \textsc{4-gram} language model, and were compiled solely using the \acro{ami} data. The column \acro{size} refers to an approximate number of differentiable parameters in the respective neural architectures.}
	\begin{tabular}[t]{l|c|c|c|r}
		\textsc{architecture} & \textsc{dev} & \textsc{eval} & \textsc{lm} & \textsc{size}\\
		\hline
		
		\hline
		\multicolumn{5}{l}{\textsc{a. raw speech baselines (adaptive filters)}}\\
		\hline
		\textsc{var. parznets (10 x cnn1d)} & $\textbf{24.7}$ & $\textbf{25.7}$ & \textsc{\fontsize{7}{10}\selectfont 3-gram} & $17.4$~M \\
		\textsc{det. parznets (10 x cnn1d)} & $25.0$ & $26.4$ & \textsc{\fontsize{7}{10}\selectfont 3-gram} & $8.7$~M \\
		\textsc{var. parznets (8 x cnn1d)} &  $25.1$ &  $26.4$ & \textsc{\fontsize{7}{10}\selectfont 3-gram} & $19.0$~M\\
		\textsc{det. parznets (8 x cnn1d)} &  $25.9$ &  $27.7$ & \textsc{\fontsize{7}{10}\selectfont 3-gram} & $9.5$~M\\
		\textsc{sincnet}~\cite{ami-sincnet} & $28.0$ & $30.2$ & \textsc{\fontsize{7}{10}\selectfont 3-gram} & $9.0$~M\\
		\textsc{multi-span-dnn}~\cite{multispan} & $27.2$ & $29.3$ & \textsc{\fontsize{7}{10}\selectfont 4-gram} & $4.7$~M\\
		\hline
		\multicolumn{5}{l}{\textsc{b. standard features (non-adaptive filters)}}\\
		\hline
		\textsc{fbank-mlp}~\cite{multispan} & $28.3$ & $31.1$ & \textsc{\fontsize{7}{10}\selectfont 4-gram} & $3.0$~M\\
		\textsc{fmllr-mlp} & $26.0$ & $27.1$ & \textsc{\fontsize{7}{10}\selectfont 3-gram} & $8.5$~M\\
		\textsc{tdnn}~\cite{Peddinti2015ATD} & $25.3$ & $26.0$ & \textsc{\fontsize{7}{10}\selectfont 3-gram} & $7.7$~M\\
		\hline
		
		\hline
	\end{tabular}
	\label{tbl:ami}\vspace{-1ex}
\end{table}

\subsection{How does the approach fare relative to state-of-the-art raw waveform baselines on \textsc{ami-ihm}?}

\textsc{ami-ihm} is a conversational speech dataset with approximately $80$ hours of speech, recorded using individual headset microphones. The alignments were generated using the Kaldi recipe configured with $3,984$ \textsc{hmm} state ids. Table~{\ref{tbl:ami}} summarizes our result relative to relevant baselines on this dataset. 

We have first compared variational \textsc{parznets} with $8$ and $10$ convolutional layers to two recently published raw waveform approaches for this task: multi-span raw waveform models~{\cite{multispan}} and \textsc{sincnet}~{\cite{ami-sincnet}}. Our empirical results show that variational \textsc{parznets} advance the state-of-the-art in waveform-based acoustic models on this dataset, with over 12\% relative improvement in \textsc{wer} compared to these baselines. 
Moreover, we also compare to deep time-delay neural networks~\cite[\textsc{tdnn},][]{Peddinti2015ATD} based on \textsc{fbank} features (considered to be the state-of-the-art feedforward model on this dataset) and show that variational inference coupled with a \textsc{parznets} architecture (\textsc{10 x cnn1d}) can outperform that approach. We note here that we have not used any data augmentation or $\mathrm{i}$-$\mathrm{vectors}$ in our experiments, both techniques which could be combined with our approach and are known to further improve the accuracy on this dataset.

Finally we note that our experiments were conducted using a cross entropy (\textsc{ce}) loss function.  Experiments using a sequence discriminative approach (\textsc{lf-mmi}) indicate that the \textsc{wer}s could be further lowered -- Povey et al~{\cite{povey16}} indicated that using \textsc{lf-mmi} in place of \textsc{ce} can reduce the error rate by about $10\%$ relative, 
and more recently a regularised \textsc{lf-mmi} training with significant data augmentation (6x) resulted in a \textsc{wer} of $18.0\%$ on this task~{\cite{kanda2018lattice}}.

\section{Discussion}
\label{sec:discussion}

This section discusses some of the model choices and assumptions made by our approach. We also address the empirical evaluation and the ablation studies that we have performed to discern the effects of individual components of our approach.

The proposed approach employs a variational family of univariate Gaussian distributions, known as the mean field assumption.  While such a variational family might be perceived as overly simplistic, recent work~\cite{Farquhar2019} has demonstrated that deep Bayesian/stochastic neural networks equipped with univariate Gaussian distributions can build complex covariance structures through multiple layers. The proposed neural architecture combines $8$-$10$ convolutional layers with multi-layer perceptrons and, thus, provides sufficient depth.

The main reason for selecting the probabilistic formulation of the neural architecture is to enforce the bounded weight property across the network and, thus, allow for learning of a robust acoustic model with a good Lipschitz constant. Variational inference alone, however, is not necessary to guarantee bounded weights across the neural network. That property will depend on the choice of prior function and holds for the Gaussian and scale-mixture priors. For the log-scale uniform prior, Section~\ref{subsec:svi-priors} provides a brief discussion and reference to relevant related work where it has been demonstrated that learning with that prior amounts to performing penalized log-likelihood estimation, with the Kullback--Leibler divergence term responsible for regularization. Moreover, the dropout regularization technique~\cite{SrivastavaHKSS14} can be theoretically justified as variational inference with the log-scale uniform prior. Hence, the proposed approach exploits means to generalize the most frequently used regularization method for neural networks. Our experiments, however, demonstrate that Gaussian and scale-mixture priors do not provide a good inductive bias for waveform-based acoustic models. Future work will explore the potential of more complex prior functions.

In our ablation study (see Section~\ref{subsec:exp-filter-learning}), we have compared the effectiveness of two identical architectures, one with modulation filter learning and the other with a priori fixed or non-adaptive filters. Our empirical results indicate that filter learning can be statistically significantly more effective than non-adaptive filters. Moreover, Fig.~\ref{fit:filters} shows that modulation frequencies converge to different distributions for different learning tasks and this is yet another indication that non-adaptive filters do not provide a universally optimal inductive bias. When evaluating the effectiveness of the approach relative to standard features such as \acro{fbank} and \acro{mfcc} one should bear in mind that different feature representations require different neural architectures and inductive biases for state-of-the-art results. Moreover, there is a significant difference in the dimension of the inputs to neural networks operating with raw waveforms on the one hand and \acro{fbank} or \acro{mfcc} features on the other, because of the aggressive compression performed by the latter. In addition to this, neural networks operating with statically extracted features typically encode more information into the training process by means of speaker and utterance level normalizations, which are known to improve the performance of acoustic models. To make the comparison between different feature representations fair, we have decided to compare our approach to state-of-the-art feedforward architectures operating in low-dimensional feature spaces. Tables~\ref{tbl-timit-ref} and~\ref{tbl-aurora4-ref} indicate a competitive performance of our approach relative to state-of-the-art baselines based on statically extracted features. Moreover, the approach is more effective than any other waveform-based approach and in this sense advances the state-of-the-art.

We conclude with a reference to the selected filterbank, which is simple to implement and provides the band-pass properties required to establish the Lipschitz continuity of the waveform-based operator mapping. The parametrization allows for an independent control over bandwidth and modulation frequency, which is sufficient to emulate a sub-band decomposition as in  standard statically extracted features. In Table~\ref{tbl-timit-ref} (see \acro{raw speech cnn} and \acro{end-to-end cnn}), we have compared to deep convolutional networks that employ modulation filter learning with a standard non-parametric convolutional layer. Our empirical results indicate that the strong inductive bias encoded via a parametric convolutional layer can lead to more effective acoustic models, especially in low-resource settings.

\section*{Conclusion}
\label{sec:conclusion}

We have outlined a principled framework for learning effective waveform-based acoustic models. The framework combines stochastic variational inference with a Lipschitz continuous architecture/operator that learns to gradually extract relevant features. The approach operates directly in the waveform domain to avoid potential information loss inherent to standard feature extraction techniques such as \acro{mfcc} and \acro{fbank} coefficients. In our experiments, the approach outperforms recently proposed architectures for waveform-based speech recognition (e.g., \acro{sincnet}) as well as a relevant deep convolutional networks for learning of robust acoustic models using \acro{fbank} features~\cite{vdcnn}. Moreover, our empirical results show that the proposed approach allows for learning of effective acoustic models using relatively small datasets. Our future work will explore the potential of stochastic recurrent architectures operating in the waveform domain as well as different priors that could further improve the inductive bias via the regularization mechanism provided by the Kullback--Leibler divergence term.
To the best of our knowledge, this is the first time that a variational approach has achieved results competitive with state-of-the-art on continuous speech recognition.

\appendices

\section{Training Procedure}
\label{app:exps}

In all the experiments, the minibatch size was set to $256$ samples. For our deterministically trained baselines, we tried two batch sizes, $256$ and $128$, and report the better of the two error rates in our tables. 
The feature extraction parameters involving Parzen filters and convolution layers that synthesize features across filtered signals were optimized using the \acro{rmsprop} algorithm~\cite{Tieleman2012} with initial learning rate $0.0008$. The fully connected blocks were optimized using the standard stochastic gradient descent with initial learning rate $0.08$. This combination of optimization algorithms (with all the blocks trained jointly) has been found to be the most effective, confirming the findings in~\cite{sincnet}. Alternative algorithms that were tried and found to be too aggressive (providing lower training error but worse generalization) were \acro{adam}~\cite{adam}, \acro{nadam}~\cite{nadam} and \acro{sgd} with momentum. Here, it is important to note that the conclusions of our ablation studies were consistent under changes to the optimization algorithm.
The learning rates were decreased by a factor of $\nicefrac{1}{2}$ if at the end of an epoch the relative improvement in validation error was below a specified threshold (e.g., $0.1\%$ for the frame classification error). Moreover, if the validation error degraded then training was continued using the model from the previous epoch (with learning rates again decreased by a factor $\nicefrac{1}{2}$). We terminate the training process after at most $25$ epochs or upon observing no improvement in the validation error for $3$ successive epochs.

In previous work~\cite{Sonderby16,molchanov17a} it was established that, for some priors, stochastic variational inference tends to trim too many parameters in the early stages of the training. To address this issue it was proposed~\cite{Sonderby16} to rescale the Kullback--Leibler regularization term with a hyperparameter $\rho_t$ such that $\rho_{t+1}=\min\{1,\rho_t+c\}$ with $\rho_0=0$ and some constant $0<c<1$ (e.g., $c=0.2$), and where $t$ denotes the epoch number (starting from $t=0$). We followed this heuristic in all of our experiments and observed an improvement in accuracy. Following the findings in~\cite{may19}, we also considered two notions of validation error in our preliminary experiments (omitted here for brevity) 
classification error of raw-speech frames and entropy regularized log-loss~\cite{may19}. The empirical results from~\cite{may19} indicate that the latter error correlates better with the token error rate of continuous speech recognition. Indeed, our best results were obtained using the entropy regularized log-loss as the validation objective. Just as in~\cite{blundell15}, we observed an improvement in accuracy for models trained using batch-specific importance weighting of the divergence term. However, the cooling schedule proposed in~\cite[][Eq. $9$]{blundell15} was too strong for the datasets considered here because of the much larger number of batches. To address this, we replaced base $2$ proposed in~\cite{blundell15} with another constant, computed such that the minimal importance weight is equal to machine precision for $32$-bit floating point arithmetics. In addition to these findings we also observed that in some cases the optimization (overly) focuses on the maximization of the log-likelihood for the already correctly classified speech frames. To mitigate this and ensure that the optimization objective is always bounded, we transformed softmax probabilities (denoted with $p$) by
\begin{align}
\log p \quad \rightarrow \quad \log \brackets{\brackets{1 - 2 \kappa} p + \kappa} \ ,
\end{align}
with $\kappa$ denoting a small jitter constant (e.g.\ $\kappa = 10^{-8}$).

\section*{Acknowledgments}

This work was supported in part by EPSRC grant EP/R012067/1. The authors would also like to thank Steve Renals and Peter Bell for valuable discussions and comments that have improved the manuscript. The Kaldi alignments were generated with the help of Erfan Loweimi and Neethu Joy.

\ifCLASSOPTIONcaptionsoff
  \newpage
\fi



{
\small
\bibliographystyle{IEEEtran}
\bibliography{parznets}
}
\end{document}